\documentclass[11pt, fleqn]{article}
\usepackage{fancyhdr}
\usepackage[ruled, oldcommands, linesnumbered]{algorithm2e}
\usepackage{amsmath}
\usepackage{amsthm}
\usepackage{amssymb}
\usepackage{float}
\usepackage{amsfonts}
\usepackage{tikz}
\usepackage[loose]{subfigure}
\usepackage[margin=1in]{geometry}
\usepackage[hidelinks]{hyperref}
\usepackage[normalem]{ulem}
\usepackage{colortbl}
\usepackage{stmaryrd}
\usepackage{hhline}
\usepackage{comment}
\usepackage{appendix}
\newtheorem{thm}{Theorem}
\newtheorem{definition}[thm]{Definition}

\newtheorem{lem}[thm]{Lemma}

\newtheorem{cor}[thm]{Corollary}
\newtheorem{claim}[thm]{Claim}
\newtheorem{fact}[thm]{Fact}
\newtheorem{concl}[thm]{Conclusion}
\newtheorem{subclaim}[thm]{Sub Claim}
\newcommand{\note}[1]{
}

\date{}
\author{Keki M. Burjorjee\\ Zite, Inc.\\487 Bryant St.\\San Francisco, CA, USA\\kekib@cs.brandeis.edu}

\title{The Fundamental Learning Problem that Genetic Algorithms with Uniform Crossover Solve Efficiently and Repeatedly\\ As Evolution Proceeds} 

\begin{document}
\maketitle
\begin{abstract}This paper establishes theoretical bonafides for \emph{implicit concurrent multivariate effect evaluation}---\href{http://blog.hackingevolution.net/2013/03/24/implicit-concurrency-in-genetic-algorithms/}{implicit concurrency}\footnote{\url{http://bit.ly/YtwdST}} for short---a broad and versatile computational learning efficiency thought to underlie general-purpose, non-local, noise-tolerant optimization in genetic algorithms with uniform crossover (UGAs). We demonstrate that implicit concurrency is indeed a form of efficient learning by showing that it can be used to obtain close-to-optimal bounds on the time and queries required to approximately correctly solve a constrained version $(k=7, \eta=1/5)$ of a recognizable computational learning problem: learning parities with noisy membership queries. We argue that a UGA that treats the noisy membership query oracle as a fitness function can be straightforwardly used to approximately correctly learn the essential attributes in $O(\log^{1.585} n)$ queries and $O(n \log^{1.585} n)$ time, where $n$ is the total number of attributes. Our proof relies on an accessible symmetry argument and the use of statistical hypothesis testing to reject a global null hypothesis at the $10^{-100}$ level of significance.  It is, to the best of our knowledge, the first relatively rigorous identification of efficient computational learning in an evolutionary algorithm on a non-trivial learning problem.\end{abstract}

\section{Introduction}
We recently hypothesized \cite{Burjorjee:2013:EOG:2460239.2460244} that an efficient form of computational learning underlies general-purpose, non-local, noise-tolerant optimization in genetic algorithms with uniform crossover (UGAs). The hypothesized computational efficiency, \emph{implicit concurrent multivariate effect evaluation}---\emph{implicit concurrency} for short---is broad and versatile, and carries significant implications for efficient large-scale, general-purpose global optimization in the presence of noise, and, in turn, large-scale machine learning. In this paper, we describe implicit concurrency and explain how it can power general-purpose, non-local, noise-tolerant optimization. We then establish that implicit concurrency is a bonafide form of efficient computational learning by using it to obtain close to optimal bounds on the query and time complexity of an algorithm that solves a constrained version of a problem from the computational learning literature: learning parities with a noisy membership query oracle \cite{UehTsuWeg00, feldman2007attribute}. 

\section{Implicit Concurrenct Multivariate Effect Evaluation}

First, a brief primer on schemata and schema partitions \cite{Mitchell:1996:IGA}: Let $ S=\{0,1\}^n$ be a search space consisting of binary strings of length $ n$ and let $ \mathcal I$ be some set of indices between $ 1$ and $ n$, i.e.  $ \mathcal I\subseteq\{1,\ldots, n\}$. Then  $ \mathcal I$ represents a partition of $ S$ into $ 2^{|\mathcal I|}$ subsets called schemata (singular schema) as in the following example: Suppose $ n=5$, and $ \mathcal I=\{1,2,4\}$, then $ \mathcal I$ partitions $ S$ into eight schemata:
\begin{center}\begin{ttfamily}
\begin{tabular}{c|c|c|c|c|c|c|c}

00*0* & 00*1* & 01*0* & 01*1* & 10*0* & 10*1* & 11*0* & 11*1*\\
\hline
00000 & 00010 & 01000 & 01000 & 01010 & 10000 & 10010 & 11010\\
00001 & 00011 & 01001 & 01001 & 01011 & 10001 & 10011 & 11011\\
00100 & 00110 & 01100 & 01100 & 01110 & 10100 & 10110 & 11110\\
00101 & 00111 & 01101 & 01101 & 01111 & 10101 & 10111 & 11111
\end{tabular}

\end{ttfamily}
\end{center}

\noindent where the symbol $*$ stands for 'wildcard'. Partitions of this type are called schema partitions. As we've already seen, schemata can be expressed using templates, for example,  $ 10*1*$. The same goes for schema partitions. For example $ \#\#*\#*$ denotes the schema partition represented by the index set $ \{1,2,4\}$; the one shown above. Here the symbol $ \#$ stands for 'defined bit'. The \emph{fineness order} of a schema partition is simply the cardinality of the index set that defines the partition, which is equivalent to the number of $\#$ symbols in the schema partition's template (in our running example, the fineness order is 3). Clearly, schema partitions of lower fineness order are coarser than schema partitions of higher fineness order.

We define the \emph{effect} of a schema partition to be the variance\footnote{We use variance because it is a well known measure of dispersion. Other measures of dispersion may well be substituted here without affecting the discussion} of the average fitness values of the constituent schemata under sampling from the uniform distribution over each schema in the partition. So for example, the effect of the schema partition $ \#\#*\#*=\{00*0*\,,\, 00*1*\,,\, 01*0*$, $\, 01*1*, 10*0*\,,\,$ $10*1*\,,\, 11*0*\,,\, 11*1*\}$ is
\[ \frac{1}{8}\sum\limits_{i=0}^1\sum\limits_{j=0}^1\sum\limits_{k=0}^1(F(ij*k*)-F(****))^2\]
\noindent where the operator $F$ gives the average fitness of a schema under sampling from the uniform distribution.   

Let $ \llbracket\mathcal I\rrbracket$ denote the schema partition represented by some index set $ \mathcal I$. Consider the following intuition pump \cite{dennett2013intuition} that illuminates how effects change with the coarseness of schema partitions. For some large $n$, consider a search space $ S=\{0,1\}^n$, and let $ \mathcal I =[n]$. Then $ \llbracket \mathcal I\rrbracket$ is the finest possible partition of $ S$; one where each schema in the partition has just one point. Consider what happens to the effect of $ \llbracket\mathcal I\rrbracket$ as we remove elements from $ \mathcal I$. It is easily seen that the effect of $ \llbracket\mathcal I\rrbracket$ decreases monotonically. Why? Because we are averaging over points that used to be in separate partitions. Secondly, observe that the number of schema partitions of order $ k$ is $ {n \choose k}$. Thus, when $k\ll n$, the number of schema partitions with fineness order $k$ will grow very fast with $k$ (sub-exponentially to be sure, but still very fast). For example, when $n=10^6$, the number of schema partitions of fineness order $2,3,4$ and $5$ are on the order of $ 10^{11}, 10^{17}, 10^{22}$, and $ 10^{28}$ respectively. 

The point of this exercise is to develop the following intuition: when $n$ is large, a search space will have vast numbers of coarse schema partitions, but most of them will have negligible effects due to averaging. In other words, while coarse schema partitions are numerous, ones with non-negligible effects are rare. Implicit concurrent multivariate effect evaluation is a capacity for scaleably learning (with respect to $n$) small numbers of coarse schema partitions with non-negligible effects. It amounts to a capacity for efficiently performing vast numbers of concurrent effect/no-effect multivariate analyses to identify small numbers of interacting loci.

\subsection{Use (and Abuse) of Implicit Concurrency}
Assuming implicit concurrency is possible, how can it be used to power efficient general-purpose, non-local, noise-tolerant optimization? Consider the following heuristic: Use implicit concurrency to identify a coarse schema partition $\llbracket\mathcal I\rrbracket$ with a significant effect. Now limit future search to the schema in this partition with the highest average sampling fitness. Limiting future search in this way amounts to permanently setting the bits at each locus whose index is in $\mathcal I$ to some fixed value, and performing search over the remaining loci. In other words, picking a schema effectively yields a new, lower-dimensional search space. Importantly, coarse schema partitions in the new search space that had negligible effects in the old search space may have detectable effects in the new space. Use implicit concurrency to pick one and limit future search to a schema in this partition with the highest average sampling fitness. Recurse. 

Such a heuristic is \emph{non-local} because it does not make use of neighborhood information of any sort. It is \emph{noise-tolerant} because it is sensitive only to the \emph{average} fitness values of coarse schemata. We claim it is \emph{general-purpose} firstly, because it relies on a very weak assumption about the distribution of fitness over a search space---the existence of \emph{staggered conditional effects} \cite{Burjorjee:2013:EOG:2460239.2460244}; and secondly, because it is an example of a \emph{decimation heuristic}, and as such is in good company. Decimation heuristics such as Survey Propagation \cite{mezard2002analytic,DBLP:conf/uai/KrocSS07} and Belief Propagation \cite{Maneva:2007:NLS:1255443.1255445}, when used in concert with local search heuristics (e.g. \emph{WalkSat} \cite{selman1993local}), are state of the art methods for solving large instances of a number of NP-Hard combinatorial optimization problems close to their solvability/unsolvability thresholds.   

The hyperclimbing hypothesis \cite{Burjorjee:2013:EOG:2460239.2460244} posits that by and large, the heuristic described above \emph{is} the abstract heuristic that UGAs implement---or as is the case more often than not, \emph{misimplement} (it stands to reason that a ``secret sauce" computational efficiency that stays secret will not be harnessed fully). One difference is that in each ``recursive step", a UGA is capable of identifying \emph{multiple} (some small number greater than one) coarse schema partitions with non-negligible effects; another difference is that for each coarse schema partition identified, the UGA does not always pick the schema with the highest average sampling fitness. 

\subsection{Needed: The Scientific Method, Practiced With Rigor}
Unfortunately, several aspects of the description above are not crisp. (How small is small? What constitutes a ``negligible" effect?) Indeed, a formal statement of the above, much less a formal proof, is difficult to provide. Evolutionary Algorithms are typically constructed with biomimicry, not formal analyzability, in mind. This makes it difficult to formally state/prove complexity theoretic results about them without making simplifying assumptions that effectively neuter the algorithm or the fitness function used in the analysis.  

\note{Recombinative EAs are especially hard to analyze}

We have argued previously \cite{Burjorjee:2013:EOG:2460239.2460244} that the adoption of the scientific method \cite{logicofscientificdisc} is a necessary and appropriate response to this hurdle. Science, \emph{rigorously practiced}, is, after all, the foundation of many a useful field of engineering. A hallmark of a rigorous science is the ongoing making and testing of predictions. Predictions found to be true lend credence to the hypotheses that entail them. The more unexpected a prediction (in the absence of the hypothesis), the greater the credence owed the hypothesis if the prediction is validated \cite{logicofscientificdisc,conjrefut}. 

The work that follows validates a prediction that is straightforwardly entailed by the hyperclimbing hypothesis, namely that a UGA that uses a noisy membership query oracle as a fitness function should be able to efficiently solve the learning parities problem for small values of $k$, the number of essential attributes, and non-trivial values of $\eta$, where $0<\eta<1/2$ is the probability that the oracle makes a classification error (returns a 1 instead of a 0, or vice versa). Such a result is completely unexpected in the absence of the hypothesis.

\subsection{Implicit Concurrency $\not =$ Implicit Parallelism}
Given its name and description in terms of concepts from schema theory, implicit concurrency bears a superficial resemblance to \emph{implicit parallelism}, the hypothetical ``engine" presumed, under the beleaguered \emph{building block hypothesis} \cite{Goldberg:1989:GAS, reeves2003gap}, to power optimization in genetic algorithms with strong linkage between genetic loci. The two hypothetical phenomena are emphatically not the same. We preface a comparison between the two with the observation that strictly speaking, implicit concurrency and implicit parallelism pertain to different kinds of genetic algorithms---ones with tight linkage between genetic loci, and ones with no linkage at all. This difference makes these hypothetical engines of optimization non-competing from a scientific perspective. Nevertheless a comparison between the two is instructive for what it reveals about the power of implicit concurrency.  

The \emph{unit of implicit parallel evaluation} is a schema $h$ belonging to a coarse schema partition $\llbracket \mathcal I \rrbracket$ that satisfies the following \emph{adjacency constraint}: the elements of $\mathcal I$ are adjacent, or close to adjacent (e.g. $\mathcal I = \{439, 441,442, 445\}$). The \emph{evaluated characteristic} is the average fitness of samples drawn from $h$, and the \emph{outcome of evaluation} is as follows: the frequency of $h$ rises if its evaluated characteristic is greater than the evaluated characteristics of the other schemata in $\llbracket \mathcal I \rrbracket$. 

The \emph{unit of implicit concurrent evaluation}, on the other hand, is the coarse schema partition $\llbracket \mathcal I \rrbracket$, where the elements of $\mathcal I$ are unconstrained. The \emph{evaluated characteristic} is the effect of $\llbracket \mathcal I \rrbracket$, and the \emph{outcome of evaluation} is as follows: if the effect of $\llbracket \mathcal I \rrbracket$ is non negligible, then a schema in this partition with an above average sampling fitness goes to fixation, i.e. the frequency of this schema in the population goes to 1.

Implicit concurrency derives its superior power viz-a-viz implicit parallelism from the absence of an adjacency constraint. For example, for chromosomes of length $n$, the number of schema partitions with fineness order 7 is ${n\choose 7}\in\Omega(n^7)$ \cite{clr}. The enforcement of the  adjacency constraint, brings this number down to $O(n)$. Schema partition of fineness order 7 contain a constant number of schemata (128 to be exact). Thus for fineness order 7, the number of units of evaluation under implicit parallelism are in $O(n)$, whereas the number of units of evaluation under implicit concurrency are in $O(n^7)$.  
 
\section{The Learning Model}
For any positive integer $n$, let $[n]$ denote the set $\{1,\ldots,n\}$. For any set $K$ such that $|K|<n$ and any binary string $x\in\{0,1\}^n$, let $\pi_K(x)$ denote the string $y\in\{0,1\}^{|K|}$ such that for any $i\in[\,|K|\,], y_i=x_j$ iff $j$ is the $i^{th}$ smallest element of $K$ (i.e. $\pi_K$ strips out bits in $x$ whose indices are not in $K$). An \emph{essential attribute oracle with random classification error} is a tuple $\phi = \langle k,f,n,K, \eta\rangle$ where $n$ and $k$ are positive integers, such that $|K|=k$ and $k< n$, $f$ is a boolean function over $\{0,1\}^k$ (i.e. $f:\{0,1\}^k\rightarrow \{0,1\}$), and $\eta$, the random classification error parameter, obeys  $0<\eta<1/2$. For any input bitstring $x$, 
\[\phi(x) = \left\{ \begin{array}{l}\phantom{\neg}f(\pi_K(x)) \textrm{ with probability } 1-\eta\\ \neg f(\pi_K(x))  \textrm{ with probability } \eta\end{array} \right.
\]
Clearly, the value returned by the oracle depends only on the bits of the attributes whose indices are given by the elements of $K$. These attributes are said to be \emph{essential}. All other attributes are said to be \emph{non-essential}. The \emph{concept space} $\mathcal C$ of the learning model is the set $\{0,1\}^n$ and the \emph{target concept} given the oracle is the element $c^*\in \mathcal C$ such that for any $i\in[n]$, $c^*_i=1 \iff i \in K$. The hypothesis space $\mathcal H$ is the same as the concept space, i.e.  $\mathcal H=\{0,1\}^n$. 

\begin{definition}[Approximately Correct Learning] Given some positive integer $k$, some boolean function $f:\{0,1\}^k\rightarrow\{0,1\}$, and some random classification error parameter $0<\eta<1/2$, we say that the learning problem $\langle k, f, \eta \rangle$ can be \emph{approximately correctly solved} if there exists an algorithm $\mathcal A$ such that for any oracle $\phi=\langle n, k, f, K, \eta \rangle$ and any $0<\epsilon<1/2$, $\mathcal A^\phi(n,\epsilon)$ returns a hypothesis $h\in\mathcal H$ such that  $P(h\not=c^*)<\epsilon$, where $c^*\in\mathcal C$ is the target concept. 
\end{definition}

\section{Our Result and Approach}
Let $\oplus_k$ denote the parity function over $k$ bits. For $k=7$ and $\eta = 1/5$, we give an algorithm that  approximately correctly learns $\langle k\,,\,\oplus_k\,,\, \eta\rangle$ in $O(\log^{1.585}n )$ queries and $O(n \log^{1.585}n)$ time. Our argument relies on the use of hypothesis testing to reject two null hypotheses, each at a Bonferroni adjusted significance level of $10^{-100}/2$. In other words, we rely on a hypothesis  testing based rejection of a global null hypothesis at the $10^{-100}$ level of significance. In layman's terms, our result is based on conclusions that have a 1 in $10^{100}$ chance of being false.

\begin{figure}[t!]\begin{center}\quad\,\,\,\emph{n}\\
\begin{ttfamily}\small
\renewcommand{\arraycolsep}{1pt}
$\,\,\,\,\,\,\,\,\,\overbrace{\phantom{\begin{array}{l*{1}{>{\centering}p{16pt}|} >{\columncolor[gray]{.85}\centering}p{16pt}|*{1}{>{\centering}p{16pt}|}>{\columncolor[gray]{.85}\centering}p{16pt}|*{3}{>{\centering}p{16pt}|}>{\columncolor[gray]{.85}\centering}p{16pt}|*{8}{>{\centering}p{16pt}|}>{\centering}p{16pt}c}
                                                                                                      1&0&1&1&0&1&0&0&1&0&1&0&1&0&0&\ldots&1&0&\\
\end{array}}}$
\vskip -.2cm
$m\left\{\begin{array}{*{2}{>{\centering}p{16pt}|} >{\columncolor[gray]{.85}\centering}p{16pt}|*{1}{>{\centering}p{16pt}|}>{\columncolor[gray]{.85}\centering}p{16pt}|*{3}{>{\centering}p{16pt}|}>{\columncolor[gray]{.85}\centering}p{16pt}|*{8}{>{\centering}p{16pt}|}>{\centering}p{16pt}c}

1&0&1&1&0&1&0&0&1&0&1&0&1&0&0&$\cdots$&1&0&\\
0&0&1&0&1&0&0&1&1&0&1&0&0&1&0&$\cdots$&0&0&\\
1&0&0&1&0&0&1&1&1&0&1&0&1&1&0&$\cdots$&1&0&\\
1&1&0&0&1&0&1&0&0&0&1&1&0&1&0&$\cdots$&0&1&\\
1&0&0&0&0&1&0&1&0&0&1&0&0&1&0&$\cdots$&1&0&\\
0&0&0&1&0&1&0&0&1&1&0&1&0&0&0&$\cdots$&0&0&\\
1&1&0&1&0&0&1&1&1&0&1&1&0&1&0&$\cdots$&1&0&\\

$\vdots$&$\vdots$&$\vdots$&$\vdots$&$\vdots$&$\vdots$&$\vdots$&$\vdots$&$\vdots$&$\vdots$&$\vdots$&$\vdots$&$\vdots$&$\vdots$&$\vdots$&$$\cdots$$&$\vdots$&$\vdots$&\\
$\vdots$&$\vdots$&$\vdots$&$\vdots$&$\vdots$&$\vdots$&$\vdots$&$\vdots$&$\vdots$&$\vdots$&$\vdots$&$\vdots$&$\vdots$&$\vdots$&$\vdots$&$$\cdots$$&$\vdots$&$\vdots$&\\

1&0&0&0&0&1&0&1&0&0&1&0&0&1&0&$\cdots$&1&0&\\
0&0&1&0&1&0&0&1&1&0&1&0&0&1&0&$\cdots$&0&0&\\
\end{array}\right.$

\end{ttfamily}
\end{center}
\caption{\label{hyppop}A hypothetical population of a UGA with a population of size $m$ and bitstring chromosomes of length $n$. The $3^{\textrm{rd}}$, $5^{\textrm{th}}$, and $9^{\textrm{th}}$ loci of the population are shown in grey.}

\end{figure}

While approximately correct learning lends itself to straightforward comparisons with other forms of learning in the computational learning literature, the following, weaker, definition of learning more naturally captures the computation performed by implicit concurrency.   

\begin{definition}[Attributewise $\epsilon$-Approximately Correct Learning] Given some positive integer $k$, some boolean function $f:\{0,1\}^k\rightarrow\{0,1\}$, some random classification error parameter $0<\eta<1/2$, and some $0<\epsilon<1/2$, we say that the learning problem $\langle k, f, \eta \rangle$ can be \emph{attributewise $\epsilon$-approximately correctly solved} if there exists an algorithm $\mathcal A$ such that for any oracle $\phi=\langle n, k, f, K, \eta \rangle$, $\mathcal A^\phi(n)$ returns a hypothesis $h\in\mathcal H$ such that for all $i\in [n]$, $P(h_i\not=c^*_i)<\epsilon$, where $c^*\in\mathcal C$ is the target concept. \end{definition}

Our argument is comprised of two parts. In the first part, we rely on a symmetry argument and the rejection of two null hypotheses, each at a Bonferroni adjusted significance level of $10^{-100}/2$, to conclude that a UGA can attributewise $\epsilon$-approximately correctly solve the learning problem $\langle k=7, f = \oplus_7, \eta =1/5\rangle$ in $O(1)$ queries and $O(n)$ time. 

In the second part (relegated to Appendix A)  we use recursive three-way majority voting to  show that for any $0<\epsilon<\frac{1}{8}$, if some algorithm $\mathcal A$ is capable of attributewise $\epsilon$-approximately correctly solving a learning problem in $O(1)$ queries and $O(n)$ time, then $\mathcal A$ can be used in the construction of an algorithm capable of approximately correctly solving the same learning problem in $O(\log^{1.585}n)$ queries and $O(n \log^{1.585}n)$ time. This part of the argument is entirely formal and does not require knowledge of genetic algorithms or statistical hypothesis testing.  

\section{Symmetry Analysis Based Conclusions}

For any integer $m\in\mathbb Z^+$, let $D_m$ denote the set $\{0,\frac{1}{m}, \frac{2}{m}, \ldots, \frac{m-1}{m},1\}$. Let $V$ be some genetic algorithm with a population of size $m$ of bitstring chromosomes of length $n$. A hypothetical population is shown in Figure \ref{hyppop}.

We define the \emph{1-frequency} of some locus $i\in[n]$ in some generation $\tau$ to be the frequency of the bit 1 at locus $i$ in the population of $V$ in generation $\tau$ (in other words the number of ones in the population of $V$ at locus $i$ in generation $\tau$ divided by $m$, the size of the population). For any $i \in\{1,\ldots,n\}$, let $\mathbf{oneFreq}_{(V,\tau,i)}$ denote the discrete probability distribution over the domain $D_m$ that gives the 1-frequency of locus $i$ in generation $\tau$ of V.\footnote{Note: $\mathbf{oneFreq}_{(V,\tau,i)}$ is an unconditional probability distribution in the sense that for any $i \in\{1,\ldots,n\}$ and any generation $\tau$, if $X_0\sim \mathbf{oneFreq}_{(V,0,i)}, \ldots, X_\tau\sim \mathbf{oneFreq}_{(V,\tau,i)}$ are random variables that give the 1-frequency of locus $i$ in generations $0,\ldots,\tau$ of $V$ in some run. Then  $P(X_\tau\,|\,X_{\tau-1},\ldots,X_{0}) = P(X_{\tau})$}

Let $\phi = \langle n,k,f,K,\eta \rangle$ be an oracle such that the output of the boolean function $f$ is invariant to a reordering of its inputs (i.e., for any $\pi:[n]\rightarrow[n]$ and any element $x\in\{0,1\}$, $f(x_1, \ldots, x_k) = f(x_{\pi(1)}, \ldots, x_{\pi(k)})$) and let $W$ be the genetic algorithm described in algorithm \ref{sgapseudo} with some population size $m$, some mutation rate $p_m$, and  that uses $\phi$ as a fitness function. An appreciation of the algorithmic symmetries in effect (for the purpose at hand, one essential locus is no different from any other essential locus, and a nonessential locus is no different from any other nonessential locus; see Appendix \ref{apB}) yields the following conclusions:
\begin{concl}
 $\forall \tau \in \mathbb Z^+_0, \forall i,j$ such that $i\in K,  j\in K, \mathbf{oneFreq}_{(W, \tau,i)}=\mathbf{oneFreq}_{(W,\tau,j)}$
 \end{concl}
 \begin{concl}
$\forall \tau\in \mathbb Z^+_0, \forall i,j$ such that $i\not\in K, j\not\in K, \mathbf{oneFreq}_{(W,\tau,i)}=\mathbf{oneFreq}_{(W,\tau,j)}$
\end{concl}

For any $i \in K$, $j\not\in K$, and any generation $\tau$ we define $\mathbf{essential}_{(W,\tau)}$ and $\mathbf{nonessential}_{(W,\tau)}$ to be the probability distributions $\mathbf{oneFreq}_{(W,\tau,i)} $ and $\mathbf{oneFreq}_{(W,\tau,j)} $ respectively. Given the conclusions reached above, these distributions are well defined. A further appreciation of the algorithmic symmetries in effect (the location of the $k$ essential loci is immaterial, each non-essential locus is just ``along for the ride" and can be spliced out without affecting the 1-frequency dynamics at other loci) yields the following conclusion:

\begin{concl}
$\forall \tau \in \mathbb Z^+_0$, $\mathbf{essential}_{(W,\tau)}$ and $\mathbf{nonessential}_{(W,\tau)}$ are invariant to $n$ and $K$  
\end{concl} 

\section{Statistical Hypothesis Testing Based Conclusions}
\begin{figure}[p!]\begin{center}
\begin{minipage}{.85\textwidth}
      \begin{algorithm}[H]
\dontprintsemicolon
\SetLine
\KwIn{$m$: population size}
\KwIn{$n$: number of bits in a bitstring chromosome}
\KwIn{$\tau$: number of generations}
\KwIn{$p_m$: probability that a bit will be flipped during mutation}
\textsf{pop} $\leftarrow$ \textsc{rand}($m$,$n$) $<$ 0.5\;
\For{t $\leftarrow$ \textrm{1} to $\tau$}{
\textsf{fitnessVals} $\leftarrow$ \textsc{evaluate-fitness}(\textsf{pop})\;
\For{i $\leftarrow$\textrm{1} to $m$}{
\textsf{totalFitness} $\leftarrow$ \textsf{totalFitness} + \textsf{fitnessVals}[$i$]\;
}
\textsf{cumNormFitnessVals}[1] $\leftarrow$  \textsf{fitnessVals}[$1$]\;
\For{i $\leftarrow$\textrm{2} to $m$}{
\textsf{cumNormFitnessVals}[$i$] $\leftarrow$ \textsf{cumNormFitnessVals}[$i-1$] +\;
\qquad(\textsf{fitnessVals}[$i$]/\textsf{totalFitness})\;
   }
   \For{i $\leftarrow$ 1 to $2m$}{
    \textsf{k} $\leftarrow$ \textsc{rand}()\;
    ctr $\leftarrow$ 1\;
    \While{\textsf{\upshape{k}} $>$ \textsf{\upshape{cumNormFitnessVals}[\textsf{\upshape{ctr}}]}}{
        \textsf{ctr} $\leftarrow$ \textsf{ctr} + 1
}
\textsf{parentIndices}[$i$] $\leftarrow$ ctr
}
   \textsf{crossOverMasks} $\leftarrow$ \textsc{rand}($m, n$) $<$ 0.5\;
       \For{i $\leftarrow$ 1 to $m$}{
            \For{j $\leftarrow$ 1 to $n$}{
                 \eIf{\textsf{\upshape{crossMasks}[$i$,$j$]}$=1$}{
                \textsf{newPop}[$i,j$]$\leftarrow$ \textsf{pop[parentIndices[$i$],$j$]}\;
}{
\textsf{newPop}[$i,j$]$\leftarrow$ \textsf{pop[parentIndices[$i+m$],$j$]}\;
                     }
}
       }
       \textsf{mutationMasks} $\leftarrow$ \textsc{rand}$(m, n) < p_m$\;
       \For{i $\leftarrow$ \textrm{1} to $m$}{
            \For{j $\leftarrow$ \textrm{1} to $n$}{
                \textsf{newPop}[$i$,$j$]$\leftarrow$ \textsc{xor}(\textsf{newPop}[$i,j$], \textsf{mutMasks}[$i,j$])
            }
       }
       \textsf{pop}$\leftarrow$\textsf{newPop}
   }

\caption[boo]{\label{sgapseudo}Pseudocode for a simple genetic algorithm with uniform crossover. The population is stored in an $m$ by $n$ array of bits, with each row representing a single chromosome. \textsc{Rand()} returns a number drawn uniformly at random from the interval [0,1] and \textsc{rand}$(a,b)<c$ denotes an $a$ by $b$ array of bits such that for each bit, the probability that it is 1 is $c$.}
\normalsize
\end{algorithm}
\end{minipage}
\end{center}
\end{figure}

\begin{fact} \label{lkfjg}Let $D$ be some discrete set. For any subset $S\subseteq D$, and any independent and identically distributed random variables $X_1, \ldots, X_N$, the following statements are equivalent: 
\begin{enumerate}
\item $\forall i \in [N], P(X_i \in D\backslash S)\geq\epsilon$
\item $\forall i \in [N], P(X_i \in S)<1-\epsilon$ 
\item $P(X_1\in S \wedge \ldots \wedge X_N \in S) < (1-\epsilon)^N$
\end{enumerate}
\end{fact}

Let $\oplus_7$ denote the parity function over seven bits, let $\psi = \langle n=5, k=7,\, f = \oplus_7\,,\, K=[7], \eta=1/5\rangle$ be an oracle, and let $G^\psi$ be the genetic algorithm  described in Algorithm \ref{sgapseudo} with population size 1500 and mutation rate 0.004 that treats $\psi$ as a fitness function. Figures \ref{secondexperiment}a and \ref{secondexperiment}b show the 1-frequency of the first and last loci of $G^\psi$ over 800 generations in each of 3000 runs. Let $D^*_m$ be the set $\{x\in D_m\, |\, 0.05<x<0.95\}$. Consider the following two null hypotheses:

\[H^{\mathbf{essential}}_0 : \quad\quad\,\,\,\,\,\sum\limits_{x \in D^*_{1500}}\mathbf{essential}_{(G^\psi, 800)}(x)\geq\frac{1}{8}\]

\[H^{\mathbf{nonessential}}_0 :  \sum\limits_{x \in (D^{\phantom{*}}_{1500}\backslash D^*_{1500})}\!\!\!\!\!\!\!\!\mathbf{nonessential}_{(G^\psi, 800)}(x)\geq\frac{1}{8}\]

\begin{figure}[p!]\begin{center}
\subfigure{\includegraphics[width=.75\textwidth]{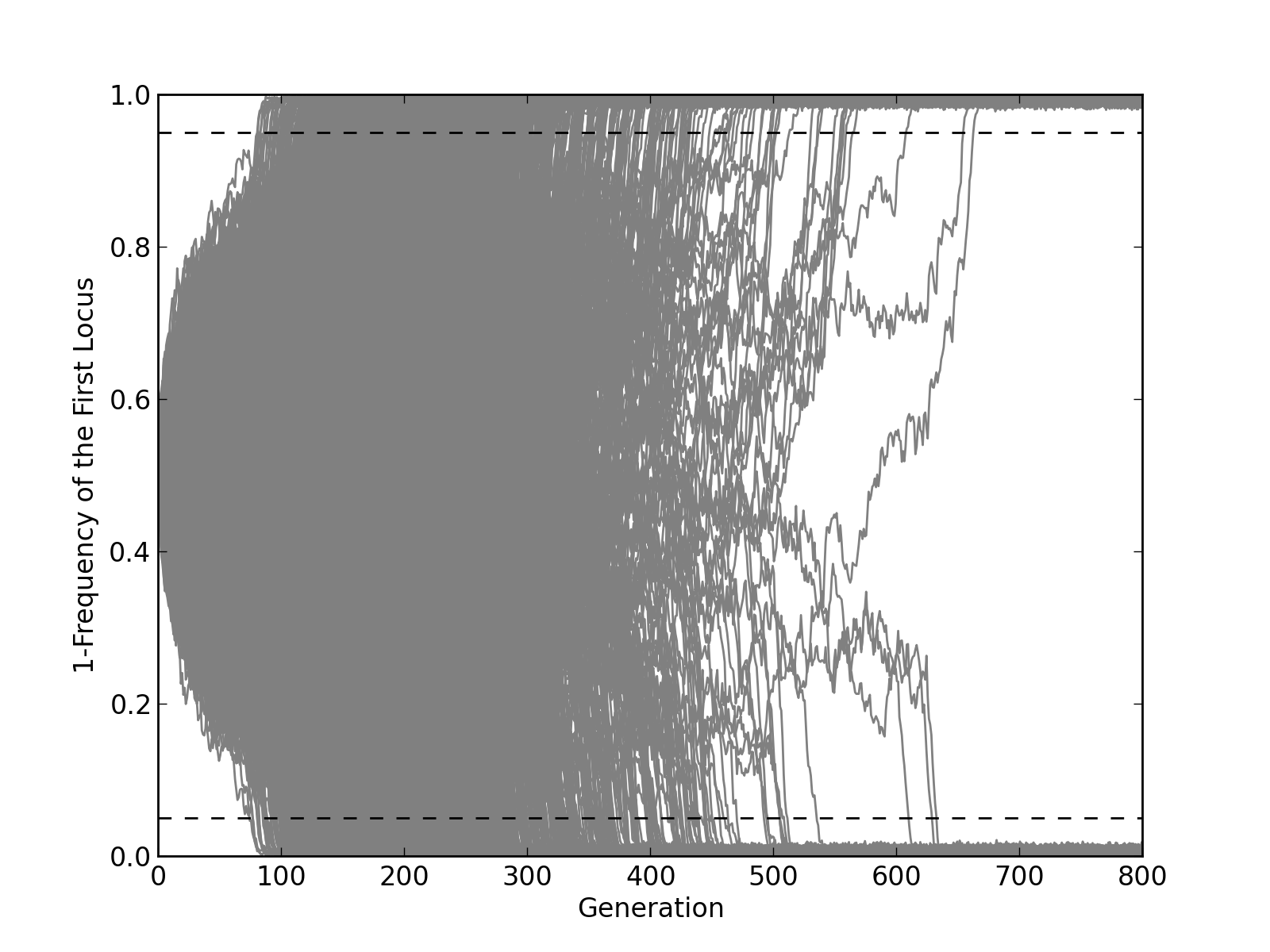}}
\subfigure{\includegraphics[width=.75\textwidth]{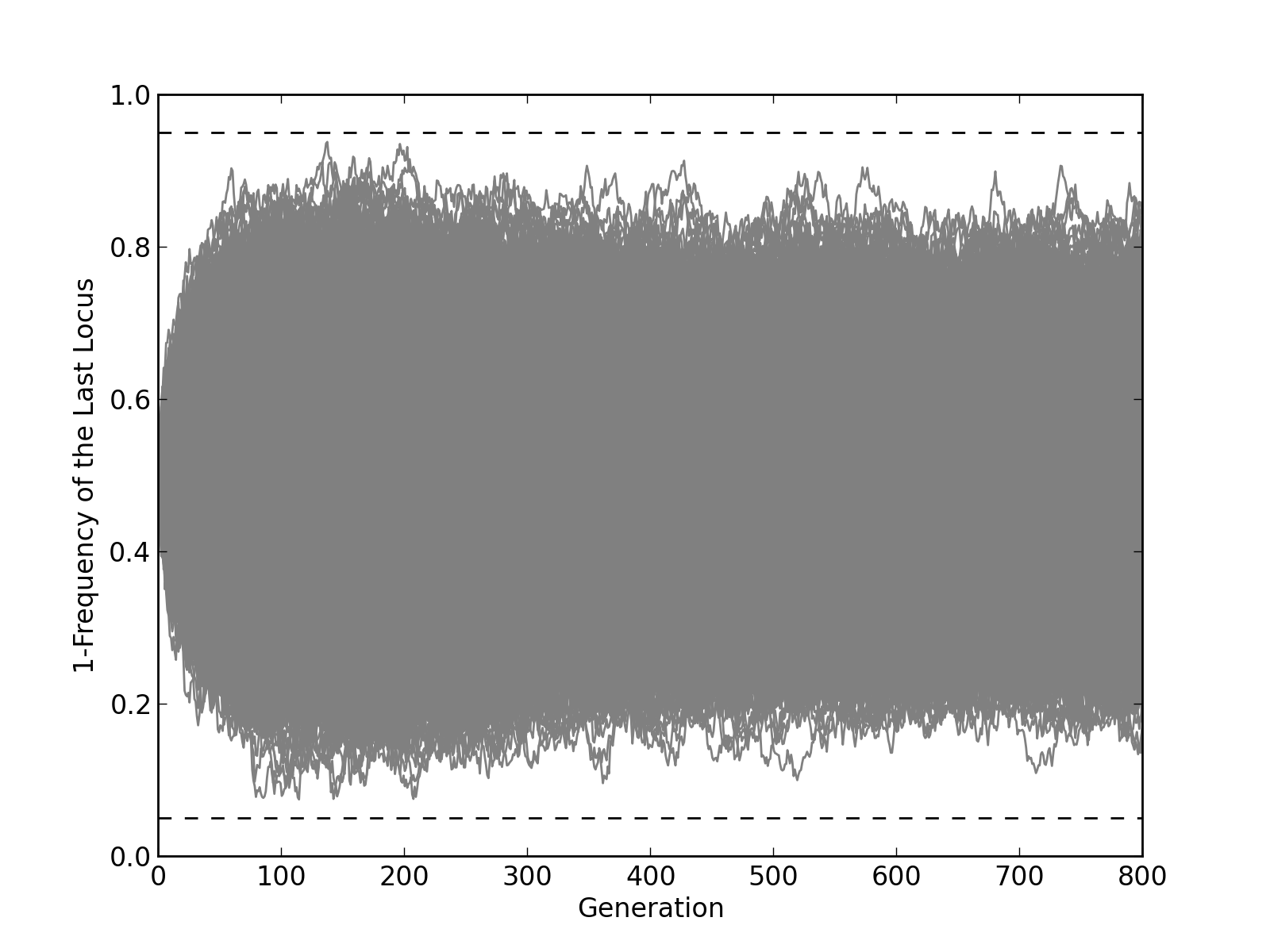}}
\end{center}
\caption{\label{secondexperiment} The 1-frequency dynamics over 3000 runs of the first (\emph{top figure}) and last (\emph{bottom figure}) loci of $G^{\psi}$, where $G$ is the genetic algorithm described in Algorithm \ref{sgapseudo} with population size 1500 and mutation rate 0.004, and $\psi$ is the oracle $\langle n=8, k=7, f = \oplus_7, K=[7],\eta=1/5\rangle$. The dashed lines mark the frequencies 0.05 and 0.95.}  
\end{figure}

\noindent where $\mathbf{essential}$ and $\mathbf{nonessential}$ are as described in the previous section. 

We seek to reject $H^{\mathbf{essential}}_0 \vee H^{\mathbf{nonessential}}_0$ at the $10^{-100}$ level of significance. 
If $H^{\mathbf{essential}}_0$ is true, then for any independent and identically distributed random variables $X_1,\ldots, X_{3000}$ drawn from distribution $\mathbf{essential}_{(G^\psi, 800)}$, and any $i\in [3000]$, $P(X_i\in D^*_{1500})\geq1/8$. 
By Lemma \ref{lkfjg}, and the observation that $D_{1500}^{\phantom{*}}\backslash(D_{1500}^{\phantom{*}}\backslash D^*_{1500}) = D^*_{1500}$, the chance that the 1-frequency of the first locus of $G^\psi$ will be in $D_{1500}^{\phantom{*}}\backslash D^*_{1500}$ in generation 800 in all 3000 runs, as seen in Figure \ref{secondexperiment}a, is less than $(7/8)^{3000}$. 
Which gives us a $p$-value less than $10^{-173}$ for the hypothesis $H^{\mathbf{essential}}_0$. 

Likewise, if $H_0^{\mathbf{nonessential}}$ is true, then for any independent and identically distributed random variables $X_1,\ldots, X_{3000}$ drawn from distribution $\mathbf{nonessential}_{(G^\phi, 800)}$, and any $i\in [3000]$, $P(X_i\in D_{1500}^{\phantom{*}}\backslash D^*_{1500})\geq1/8$. 
So by Lemma \ref{lkfjg}, the chance that the 1-frequency of the last locus of $G^\phi$ will be in $D^*_{1500}$ in generation 800 in all 3000 runs, as seen in Figure \ref{secondexperiment}b, is less than $(7/8)^{3000}$. 
Which gives us a $p$-value less than $10^{-173}$ for the hypothesis $H^{\mathbf{nonessential}}_0$. 
Both $p$-values are less than a Bonferroni adjusted critical value of $10^{-100}/2$, so we can reject the global null hypotheses $H^{\mathbf{essential}}_0 \vee H^{\mathbf{nonessential}}_0$ at the $10^{-100}$ level of significance. We are left with the following conclusions:

\begin{concl}\label{c7}
\[\quad\,\,\,\sum\limits_{x \in D^*_{1500}}\mathbf{essential}_{(G^\phi, 800)}(x)<\frac{1}{8}\]
\end{concl}
\begin{concl}\label{c8}
\[\sum\limits_{x \in (D^{\phantom{*}}_{1500}\backslash D^*_{1500})}\!\!\!\!\!\!\!\!\mathbf{nonessential}_{(G^\phi, 800)}(x)<\frac{1}{8}\] 
\end{concl}

\begin{figure}[H]
\begin{algorithm}[H]
\dontprintsemicolon
$pop$ $\leftarrow$ population of $G^\phi$ after 800 generations\; 
\For{$i \leftarrow 1$ to $n$}
{
	\eIf{ $0.05< $ 1-frequency of locus $i$ in $pop$ $< 0.95$}
	{
   	 	x[$r$][$i$] = 0
	}
	{
    	x[$r$][$i$] = 1
	}
}
\caption{\label{ggggg}$\mathcal G^\phi$}
\end{algorithm}
\end{figure}

\begin{concl} The learning problem $\langle k=7, f=\oplus_7, \eta = 1/5\rangle$ can be attributewise $\frac{1}{8}$-approximately correctly solved in $O(1)$ queries and $O(n)$ time.
\end{concl}
\begin{proof}[Argument]
For any oracle $\phi = \langle n, k = 7, f = \oplus_7, K, \eta = 1/5\rangle$ with target concept $c$, let $h$ be the hypothesis returned by the algorithm $\mathcal G^\phi$ shown in Algorithm \ref{ggggg}. It is easily seen that $\mathcal G^\phi$ runs in $O(1)$ queries and $O(n)$ time; by Conclusions \ref{c7} and \ref{c8}, for any $i\in[n]$,  $P(c_i\not = h_i)<\frac{1}{8}$

\end{proof}
That $\langle k = 7, f = \oplus_7, \eta = 1/5\rangle$ is approximately correctly learnable in $O(\log^{1.585}n)$ queries and $O(n\log^{1.585}n)$ time follows from the conclusion above and from Theorem \ref{mainthm} in the Appendix.

\section{Discussion and Conclusion} 

We used an empirico symmetry analytic proof technique with a 1 in $10^{100}$ chance of error to conclude that a genetic algorithm with uniform crossover can be used to solve the noisy learning parities problem $\langle k = 7, f = \oplus_7, \eta=1/5\rangle$ in $O(\log^{1.585}n)$ queries and $O(n\log^{1.585}n)$ time. These bounds are marginally higher than the known optimal bounds for a more difficult version of the problem in which the oracle is queried \emph{non-adaptively}: $O(\log n)$ and $O(n)$ respectively \cite{feldman2007attribute, hofmeister1999application}. Tighter bounds on the query complexity than the one obtained can be achieved using recursive majority voting with a higher branching factor (e.g. 5 instead of 3). However, for our purposes the bounds obtained suffice. The finding that a genetic algorithm with uniform crossover can straightforwardly be used to obtain close-to-optimal bounds on the queries and running time required to solve to solve $\langle 7, \oplus_7, 1/5\rangle$ is wholly unexpected and lends support to the hypothesis that implicit concurrent multivariate effect evaluation powers general-purpose, non-local, noise-tolerant optimization in genetic algorithms with uniform crossover.  

\note{Talk about generalization}

\bibliographystyle{plain}
\bibliography{refs}
\newpage
\noindent \LARGE \textbf{Appendix} \normalsize
\begin{appendix}

\section{Formal Analysis}

\begin{figure}[H]\begin{center}
\begin{minipage}{.85\textwidth} 
\begin{algorithm}[H]
\dontprintsemicolon
\SetLine
\Indp{
\SetAlgoNoLine{

	\eIf{$\ell=1$} 		
	{
		\Return $M(x_1,x_2,x_3$) \;
	}
	{
		$y_1$ = \textsc{Recursive-3-Way-Maj}$(\ell-1,(x_1, \ldots, x_{3^{\ell-1}}))$\;
		$y_2$ = \textsc{Recursive-3-Way-Maj}$(\ell-1,(x_{3^{\ell-1}+1}, \ldots, x_{2\times3^{\ell-1}}))$\; 
		$y_3$ = \textsc{Recursive-3-Way-Maj}$(\ell-1,(x_{2\times 3^{\ell-1}+1}, \ldots, x_{3^\ell}))$ \;
		\Return $M(y_1,y_2,y_3)$\;
	}	
} 
}
\caption{\label{lkdfgh}\textsc{Recursive-3-Way-Maj}($\ell, \left( x_1,\ldots, x_{3^{\ell}} \right) $) }			
\end{algorithm}	
\end{minipage}
\end{center}
\end{figure}

Let $M:\{0,1\}^3\rightarrow \{0,1\}$ denote the three way majority function that returns the mode of its three arguments. So, $M(1,1,1) = M(0,1,1) = M(1,0,1) =  M(1, 1, 0) = 1$, and $M(0,0,0) = M(1,0,0) =  M(0,1,0) = M(0,0,1) = M(0,0,0) = 0$. 

\begin{lem} \label{lkfjghkd} Let $X_1, X_2, X_3$ be independent binary random variables such that for any $i\in\{1,2,3\}$, $P(X_i=0)<\epsilon$ . Then $P(M(X_1,X_2,X_3)=0)<4\epsilon^2$
\end{lem}
\noindent\emph{Proof}: \begin{align*} P(M(X_1,X_2,X_3)=0)\, =\,&P(X_1=0 \wedge X_2 = 0 \wedge X_3 = 0) +\\
							&P(X_1=1 \wedge X_2 = 0 \wedge X_3 = 0) +\\
							&P(X_1=0 \wedge X_2 = 1 \wedge X_3 = 0) +\\
							&P(X_1=0 \wedge X_2 = 0 \wedge X_3 = 1)\\
							<\,& \epsilon^3 +3\epsilon^2\\
							<\,&4\epsilon^2
					\end{align*}
Consider Algorithm \ref{lkdfgh}. 

\begin{thm}[Recursive 3-way majority voting] \label{r3wmv}For any $\epsilon>0$ and any $\ell\in \mathbb Z^+$, let $X_1,\ldots, X_{3^\ell}$ be independent binary random variables such that $\forall i\in[3^\ell], P(X_i=0)<\epsilon$. Then $P($\textsc{Recursive-3-Way-Maj}($\ell, (X_1, \ldots, X_{3^\ell}))=0)<4^{2^\ell-1}\epsilon^{2^\ell}$  
\end{thm}
\begin{proof} The proof is by induction on $\ell$. The base case, when $\ell=1$, follows from lemma \ref{lkfjghkd}. We assume the inductive hypothesis for $\ell=k$ and prove it for $\ell=k+1$. Let $Y_1, \ldots, Y_{3}$ be random binary variables that give the values of $y_1, y_2, y_3$ in the first (top-level/non-recursive) pass of Algorithm. Three applications of the inductive hypothesis for $\ell = k$ gives us that $\forall i \in \{1,2,3\}$, $P(Y_i=0)<4^{2^k-1}\epsilon^{2^k}$. Since $X_1, \ldots, X_{3^{k+1}}$ are independent, $Y_1, \ldots, Y_{3}$ are also independent. So, by Lemma \ref{lkfjghkd}, 
\begin{align*}
P(\textsc{Recursive-3-Way-Maj}(\ell, (X_1, \ldots, X_{3^{k+1}}))=0) &< 4\left(4^{2^k-1}\epsilon^{2^k}\right)^2\\
& = 4 \left(4^{2^k-1}\right)^2\left(\epsilon^{2^k}\right)^2\\
& = 4 \left(4^{2(2^k-1)}\right)\left(\epsilon^{2(2^k)}\right)\\
& = 4 \left(4^{2^{k+1-2}}\right)\left(\epsilon^{2^{k+1}}\right)\\
& = 4^{2^{k+1}-1}\epsilon^{2^{k+1}}
\end{align*}
\end{proof}

\begin{cor} \label{r3wmv} For any $\ell\in \mathbb Z^+$, let $X_1,\ldots, X_{3^\ell}$ be independent binary random variables such that for any $ i\in[3^\ell], P(X_i=0)<\frac{1}{8}$. Then,\\ $P($\textsc{Recursive-3-Way-Majority}($\ell, (X_1, \ldots, X_{3^\ell}))=0)<1/2^{2^\ell}$  
\end{cor}
\begin{proof} Follows from Theorem  \ref{r3wmv} and the  observation that 
\begin{align*}
4^{2^\ell-1}\left(\frac{1}{8}\right)^{2^\ell}
&=4^{2^\ell-1}\left(\frac{1}{4}\right)^{2^\ell}\left(\frac{1}{2}\right)^{2^\ell}\\
&=\frac{1}{4}\left(\frac{1}{2}\right)^{2^\ell}\\
&<\frac{1}{2^{2^\ell}}\\
\end{align*}
\end{proof}

\noindent Noting that $0$ and $1$ are just labels in the above gives us the following:

\begin{cor} \label{ldfjgldj}For any $\ell\in \mathbb Z^+$, let $X_1,\ldots, X_{3^\ell}$ be independent binary random variables such that for all $i\in[3^\ell], P(X_i=1)<\frac{1}{8}$. Then,\\ $P($\textsc{Recursive-3-Way-Majority}($\ell, (X_1, \ldots, X_{3^\ell}))=1)<1/2^{2^\ell}$  
\end{cor}

\begin{figure}[t!]\begin{center}
\begin{minipage}{.85\textwidth}
\begin{algorithm}[H]
\dontprintsemicolon
\SetLine
\KwIn{$n$}
\KwIn{$\epsilon$}
$\ell \leftarrow \lceil \log_2(\log_2 n + \log_2\frac{1}{\epsilon})\rceil+1$\;
\For{$r \leftarrow 1$ to $3^\ell$}
{
    x[$r$] = $\mathcal A^\phi(n)$
}
\For{$i \leftarrow 1$ to $n$}
{
    h[$i$] = \textsc{Recursive-3-Way-Majority}($\ell$, (x[0][$i$], \ldots, x[$3^\ell$][$i$]))
}
\Return h

\caption[boo]{\label{majorityGA}$\mathcal B^\phi(n,\epsilon)$}
\normalsize
\end{algorithm}
\end{minipage}
\end{center}
\end{figure}

\begin{thm} \label{mainthm}For any $0<\epsilon<1/8$, if the learning problem $\langle k, f, \eta \rangle$ can be attributewise $\epsilon$-approximately correctly solved in $O(1)$ queries and $O(n)$ time, then $\langle k, f, \eta \rangle$ can be approximately correctly solved in $O(\log^{1.585}n)$ queries and $O(n\log^{1.585} n)$ time\end{thm}

\noindent \emph{Proof} Let $\mathcal A$ be an algorithm that attributewise $\epsilon$-approximately correctly solves $\langle k, f, \eta\rangle$ in $O(1)$ queries and $O(n)$ time. Let $\phi=\langle n, k,f, K, \eta \rangle$  be some oracle.  Consider Algorithm \ref{majorityGA} that executes $\ell =  \lceil \log_2(\log_2 n + \log_2\frac{1}{\epsilon})\rceil$ runs of the algorithm $\mathcal A^\phi$. Let $c$ be the target concept, and for any $i\in[n]$, let $H_i$ be a binary random variable that gives the value of h[$i$] (set in line 6), and let $H$ be the random variable that give the value of h. 

\begin{claim}\label{claima}  For all $i\in[n]$, $P(H_i \not=c_i)<1/2^{2^\ell}$. 
\end{claim}
\begin{proof}[Proof of Claim \ref{claima}]
For any $r\in[3^\ell]$, let $X_{(r,i)}$ be a binary random variable that gives the value of x[$r$][$i$] where x[$r$], set in line 3, is the hypothesis returned by $\mathcal A$ in run $r$.  Clearly, $X_{(1,i)},\ldots, X_{(3^\ell,i)}$ are independent. The claim follows from Corollaries  \ref{r3wmv} and \ref{ldfjgldj}, and the premise that $\langle 7, \oplus_7, 1/5\rangle$ is attributewise $\frac{1}{8}$-approximately correctly learnable by $\mathcal A$. 
\end{proof}

\begin{claim}\label{claimb} If $n/2^{2^\ell}<\epsilon$, then $P(H\not=c)<\epsilon$
\end{claim}
\begin{proof}[Proof of Claim \ref{claimb}]
The claim follows from Claim \ref{claima} and the union bound.
\end{proof}

\begin{claim}\label{claimc} $n/2^{2^\ell}<\epsilon$
\end{claim}
\begin{proof}[Proof of Claim \ref{claimc}]
\begin{align*}
\frac{n}{2^{2^\ell}} &=  \frac{n}{2^{2^{\lceil \log_2(\log_2 n + \log_2\frac{1}{\epsilon})\rceil+1}}}\\
&<\frac{n}{2^{2^{\log_2(\log_2 n + \log_2\frac{1}{\epsilon})}}}\\
&=\frac{n}{2^{\log_2 n + \log_2\frac{1}{\epsilon}}}\\
&=\frac{n}{n/\epsilon}\\
&=\epsilon
\end{align*}
\end{proof}

By claims \ref{claimb} and \ref{claimc}, $\mathcal B$ approximately correctly solves the learning problem $\langle 7, \oplus_7, 1/5\rangle$. We now consider the query complexity of $\mathcal B$. Let $q_{\mathcal A}(n)$, $q_{\mathcal B}(n)$ give the number of queries made by algorithms $\mathcal A$ and $\mathcal B$ respectively. Likewise,  let $t_{\mathcal A}(n)$, $t_{\mathcal B}(n)$ give the running time of algorithms $\mathcal A$ and $\mathcal B$ respectively. The following two claims complete the proof. 
 
\begin{claim}\label{claimd} $q_{\mathcal B}(n)\in O(\log^{1.585} n)$:
\end{claim}

\begin{claim}\label{claime} $t_{\mathcal B}(n)\in O(n\log^{1.585} n)$:
\end{claim}

\begin{proof}[Proof of Claim \ref{claimd}]
By the premise of the theorem, there exist constants $n_0, c_{\mathcal A}$ such that for all $n\geq n_0$, $q_{\mathcal A}(n)\leq c_{\mathcal A}$.  Thus, for all $n\geq n_0$,

\begin{align*}
q_{\mathcal B}(n) &\leq c_{\mathcal A}.3^\ell\\
&= c_{\mathcal A}.3^{ \lceil \log_2(\log_2 n + \log_2\frac{1}{\epsilon})\rceil+1}\\
&\leq c_{\mathcal A}.3^{ \log_2(\log_2 n + \log_2\frac{1}{\epsilon})+2}\\
&= 9c_{\mathcal A}.3^{ \log_2(\log_2 n + \log_2\frac{1}{\epsilon})}\\
\end{align*}
\noindent Taking logs to the base 2 on both sides gives
\begin{align*}
&\quad \log_2(q_{\mathcal B}(n)) \leq \log_2(9c_{\mathcal A})+\log_2(\log_2 n + \log_2\frac{1}{\epsilon}).\log_2 3\\
\Leftrightarrow&\quad q_{\mathcal B}(n) \leq 9c_{\mathcal A}.(\log_2 n + \log_2\frac{1}{\epsilon})^{\log_2 3}\\
\Leftrightarrow&\quad q_{\mathcal B}(n) \leq 9c_{\mathcal A}.(\log_2 n + \log_2\frac{1}{\epsilon})^{1.585}
\end{align*}
\end{proof}

\begin{proof}[Proof of Claim \ref{claime}]
Let $\tau_1(n)$, $\tau_2(n)$ be the time taken to execute lines 1---4 and one iteration of the for loop in lines 5--8  of $\mathcal B$, respectively. Clearly,  $t_{\mathcal B}(n) = \tau_1(n)+n.\tau_2(n)$.
 
\begin{subclaim}\label{subclaima} $\tau_1(n) \in O(n\log^{1.585} n)$

\begin{proof}[Proof of Sub Claim \ref{subclaima}]
The proof closely mirrors the proof of Claim \ref{claimd}. We include it for the sake of completeness. By the premise of the theorem, there exist constants $n_0, c_{\mathcal A}$ such that for all $n\geq n_0$, $t_{\mathcal A}(n)\leq c_{\mathcal A}$.  Thus, for all $n\geq n_0$,

\begin{align*}
\tau_1(n) &\leq c_{\mathcal A}.n.3^\ell\\
&= c_{\mathcal A}.n.3^{ \lceil \log_2(\log_2 n + \log_2\frac{1}{\epsilon})\rceil+1}\\
&\leq c_{\mathcal A}.n.3^{ \log_2(\log_2 n + \log_2\frac{1}{\epsilon})+2}\\
&= 9c_{\mathcal A}.n.3^{ \log_2(\log_2 n + \log_2\frac{1}{\epsilon})}\\
\end{align*}
\noindent Taking logs to the base 2 on both sides gives
\begin{align*}
&\quad \log_2(\tau_1(n)) \leq \log_2(9c_{\mathcal A}.n)+\log_2(\log_2 n + \log_2\frac{1}{\epsilon}).\log_2 3\\
\Leftrightarrow&\quad \tau_1(n) \leq 9c_{\mathcal A}.n.(\log_2 n + \log_2\frac{1}{\epsilon})^{\log_2 3}\\
\Leftrightarrow&\quad \tau_1(n) \leq 9c_{\mathcal A}.n.(\log_2 n + \log_2\frac{1}{\epsilon})^{1.585}
\end{align*}
\end{proof}
\end{subclaim}

\begin{subclaim}\label{subclaimb} $\tau_2(n) \in O(\log^{1.5} n)$
\end{subclaim}
\begin{proof}[Proof of Sub Claim \ref{subclaimb}]
Observe that $\tau_2(n) \leq T(3^{\log_2(\log_2 n + \log_2\frac{1}{\epsilon})+2})$, where $T$ is given by the following recurrence relation:

\[T(x) = \left\{\begin{array}{cl}3T(x/3) + x & \textrm{if }x>3\\1& \textrm{if }x=3\end{array}\right.\]

\noindent A simple inductive argument (omitted) gives us $T(x)\in O(x\log x)$. Thus, 
\begin{align*}
\tau_2(n)&\in O(3^{\log_2(\log_2 n + \log_2 \frac{1}{\epsilon})+2} (\log_2(\log_2 n + \log_2 \frac{1}{\epsilon})+2))\\
&= O(3^{\log_2(\log_2 n + \log_2 \frac{1}{\epsilon})} (\log_2(\log_2 n + \log_2 \frac{1}{\epsilon})+2))\\
&= O(3^{\log_2(\log_2 n + \log_2 \frac{1}{\epsilon})} \log_2(\log_2 n + \log_2 \frac{1}{\epsilon}))\\
&= O((\log_2 n + \log_2 \frac{1}{\epsilon})(\log_2(\log_2 n + \log_2 \frac{1}{\epsilon})))\\
&\subset O((\log_2 n + \log_2 \frac{1}{\epsilon})^{1.5})\\ 
\end{align*}
\noindent Where the last equation follows from the observation that $\log(x)<\sqrt x$ for all positive reals. 
\end{proof}
\noindent Claim \ref{claime} follows from the observation that 
\begin{align*}
t_{\mathcal B}(n) = \tau_1(n) + \tau_2(n) &\Rightarrow t_{\mathcal B}(n) \in O(n\log^{1.585}n+n\log^{1.5} n)\\ 
&\Rightarrow t_{\mathcal B}(n) \in O(n\log^{1.585}n)
\end{align*}
Where the first implication follows from Sub Claims \ref{subclaima} and \ref{subclaimb}, and the fact that for any functions $f_1\in O(g_1), f_2 \in O(g_2)$, we have that $f_1+f_2 \in O(|g_1|+|g_2|)$ and $f_1.f_2\in O(g_1.g_2)$ 
\end{proof}

\section{On Our Use of Symmetry} \label{apB}

A homologous crossover operation between two chromosomes of length $\ell$ can be modeled by a vector of $\ell$ random binary variables $\langle X_1, \ldots, X_\ell\rangle$ from which a crossover mask is sampled. Likewise, a mutation operation can be modeled by a vector of $\ell$ random binary variables $\langle Y_1, \ldots, Y_\ell\rangle$ from which a mutation mask is sampled. Only in the case of uniform crossover are the random variables $X_1, \ldots, X_\ell$ independent and identically distributed. This absence of \emph{positional bias} \cite{eshelman1989bcl} in uniform crossover constitutes a symmetry. Essentially, permuting the bits of all chromosomes using some permutation $\pi$ before crossover, and permuting the bits back using $\pi^{-1}$ after crossover has no effect on the dynamics of a UGA. If, in addition, the random variables $Y_1, \ldots, Y_\ell$ that model the mutation operator are independent and identically distributed (which is typical), and (more crucially) independent of the value of $\ell$, then in the event that the values of chromosomes at some locus $i$ are immaterial during fitness evaluation, the locus $i$ can be ``spliced out" without affecting allele dynamics at other loci. In other words, the dynamics of the UGA can be \emph{coarse-grained} \cite{CoarseGrainingFoga2007}.

These conclusions flow readily from an appreciation of the symmetries induced by uniform crossover and length independent mutation. While the use of symmetry arguments is uncommon in EC research, symmetry arguments form a crucial part of the foundations of physics and chemistry. Indeed, according to the theoretical physicist E. T. Jaynes ``almost the only known exact results in atomic and nuclear structure are those which we can deduce by symmetry arguments, using the methods of group theory" \cite[p331-332]{jaynes}. Note that the conclusions above hold true regardless of the selection scheme (fitness proportionate, tournament, truncation, etc), and any fitness scaling that may occur (sigma scaling, linear scaling etc).  ``The great power of symmetry arguments lies just in the fact that they are not deterred by any amount of complication in the details", writes Jaynes \cite[p331]{jaynes}. An appeal to symmetry, in other words, allows one to cut through complications that might hobble attempts to reason within a formal axiomatic system.

Of course, symmetry arguments are not without peril. However, when used sparingly and only in circumstances where the symmetries are readily apparent, they can yield significant insight at low cost. It bears emphasizing that the goal of foundational work in evolutionary computation is not pristine mathematics within a formal axiomatic system, but insights of the kind that allow one to a) explain optimization in current evolutionary algorithms on real world problems, and b) design more effective evolutionary algorithms.

\end{appendix}
\end {document}